\documentclass{article} 
\usepackage{iclr2026_delta,times}

\usepackage{amsmath,amsfonts,bm}

\def\eqref#1{equation~\ref{#1}}

\def\1{\bm{1}}

\DeclareMathAlphabet{\mathsfit}{\encodingdefault}{\sfdefault}{m}{sl}
\SetMathAlphabet{\mathsfit}{bold}{\encodingdefault}{\sfdefault}{bx}{n}

\newcommand{\E}{\mathbb{E}}

\usepackage{amsmath,amssymb,amsthm,bm,mathtools}
\usepackage{subcaption, graphicx}
\newcommand{\br}[1]{\left( #1 \right)}
\newcommand{\M}{\mathcal{M}}
\newcommand{\N}{\mathcal{N}}

\newcommand{\xref}{x_{\mathrm{ref}}}

\usepackage{hyperref}              
\hypersetup{
    colorlinks=true,
    linkcolor=red,
    citecolor=blue,
    urlcolor=blue
}    
\theoremstyle{plain}
\newtheorem{theorem}{Theorem}[section]
\newtheorem{proposition}[theorem]{Proposition}
\newtheorem{lemma}[theorem]{Lemma}

\theoremstyle{definition}

\newtheorem{assumption}[theorem]{Assumption}
\theoremstyle{remark}
\newtheorem{remark}[theorem]{Remark}

\title{Score-Guided Proximal Projection: A Unified Geometric Framework for Rectified Flow Editing}

\iclrfinalcopy
\author{Vansh Bansal \& James G. Scott\\
Department of Statistics and Data Sciences\\
UT Austin\\
\texttt{vansh@utexas.edu, james.scott@austin.utexas.edu}
}

\begin{document}

\maketitle

\begin{abstract}
Rectified Flow (RF) models achieve state-of-the-art generation quality, yet controlling them for precise tasks—such as semantic editing or blind image recovery—remains a challenge. Current approaches bifurcate into inversion-based guidance, which suffers from ``geometric locking" by rigidly adhering to the source trajectory, and posterior sampling approximations (e.g., DPS), which are computationally expensive and unstable. In this work, we propose Score-Guided Proximal Projection (SGPP), a unified framework that bridges the gap between deterministic optimization and stochastic sampling. We reformulate the recovery task as a proximal optimization problem, defining an energy landscape that balances fidelity to the input with realism from the pre-trained score field. We theoretically prove that this objective induces a \textit{normal contraction} property, geometrically guaranteeing that out-of-distribution inputs are snapped onto the data manifold, and it effectively reaches the posterior mode constrained to the manifold. Crucially, we demonstrate that SGPP generalizes state-of-the-art editing methods: RF-inversion is effectively a limiting case of our framework. By relaxing the proximal variance, SGPP enables ``soft guidance," offering a continuous, training-free trade-off between strict identity preservation and generative freedom.
\end{abstract}

\section{Introduction}
\label{sec:intro}

Rectified Flow (RF) models \citep{liu2022flowstraightfastlearning} have recently emerged as a powerful paradigm for generative modeling, offering high-fidelity sampling with straighter, more efficient transport trajectories than standard diffusion models \citep{song2021scorebasedgenerativemodelingstochastic}. However, harnessing these pre-trained priors for controlled inverse problems---such as semantic editing or blind image recovery---remains a non-trivial challenge. The core difficulty lies in the ``perception-distortion trade-off'': balancing \textit{fidelity} (preserving the identity or structure of the reference input) with \textit{realism} (ensuring the output remains on the learned data manifold).
Current approaches to this problem typically fall into two distinct regimes, each with fundamental limitations:
\paragraph{Inversion-Based Editing:} Exemplified by techniques like RF-Inversion \citep{rout2024semanticimageinversionediting}, these methods enforce ``hard guidance.'' They compel the editing trajectory to rigidly retrace the noise inversion path of the source image. While this preserves structure, it suffers from what we term \textit{``geometric locking''}: the inability to deviate sufficiently from the original path to accommodate significant semantic changes or correct large out-of-distribution (OOD) corruptions.
\paragraph{Posterior Sampling \& Manifold Constraints:} Methods like Diffusion Posterior Sampling (DPS) \citep{chung2024diffusionposteriorsamplinggeneral} and Manifold Constrained Gradients (MCG) \citep{chung2024improvingdiffusionmodelsinverse} attempt to solve the inverse problem by optimizing a likelihood objective $\nabla_{x_t} \log p(\xref|x_t)$. While theoretically sound, DPS relies on backpropagating through the denoising network Jacobian $\nabla_{x_t}\hat{x}_0(x_t)$, which is computationally expensive and notoriously unstable at high noise levels \cite{}. MCG attempts to stabilize this by projecting gradients onto the data manifold, but relies on explicit, approximate projections that are often brittle in practice.

In this work, we propose \textbf{Score-Guided Proximal Projection (SGPP)}, a unified framework that bridges the gap between deterministic optimization and stochastic sampling. SGPP can be viewed as a \textit{Jacobian-free implementation of MCG}: instead of computing unstable backpropagation gradients or explicit projections, we leverage the intrinsic geometry of the Rectified Flow score field.

We reformulate the recovery task as a proximal optimization problem on the time-dependent manifold. By defining a dynamic energy potential that balances a \textit{fidelity potential} (anchoring the trajectory to the input) and a \textit{generative potential} (derived from the pre-trained score field), we achieve the following:

\begin{itemize}
    \item \textbf{Geometric Stability (Normal Contraction):} We prove that the gradient flow of our proximal objective exhibits a \textit{normal contraction} property. The score field naturally decomposes into a restoring force that exponentially contracts the distance to the manifold, guaranteeing that inputs are safely projected onto valid support without the instability of DPS. We further show that our deterministic optimization algorithm converges to the manifold-constrained mode of the posterior $p(x \mid \xref)$ under the assumed likelihood.
    
    \item \textbf{Unified ``Soft Guidance'':} We reveal that state-of-the-art editing methods are effectively special limiting cases of our approach. The ``hard guidance'' of RF-Inversion corresponds to our objective when the proximal variance $\sigma_p \to 0$. By relaxing this parameter ($\sigma_p > 0$), SGPP enables \textit{``soft guidance,''} allowing the generative trajectory to deviate flexibly from the rigid inversion path to satisfy semantic constraints while remaining geometrically safe.
\end{itemize}

SGPP is \textbf{training-free}, requiring no auxiliary networks or complex distance functions. It repurposes the pre-trained score function as a geometric oracle, providing a robust, theoretically grounded solution for both blind image recovery and flexible semantic editing.

\section{Preliminaries and Notation}

\textbf{Rectified Flow.} We consider the Rectified Flow (RF) process $X_{t}=(1-t)X_{0}+tZ$ for $t\in[0,1]$, interpolating between data $X_{0}\sim p_{data}$ and noise $Z\sim\mathcal{N}(0,I)$. The probability path is generated by the vector field $v(x, t) = \E[Z - X_0 | X_t = x]$, which simplifies via Tweedie's formula to:
\begin{equation*}
    v(x, t) = -\frac{x}{1-t} - \frac{t}{1-t} \nabla_x \log p_t(x).
\end{equation*}
Sampling involves integrating the ODE $dX_t = v(X_t, t) dt$ from $t=1$ to $0$.

\textbf{Manifold Geometry.} We assume $p_{data}$ is supported on a smooth, compact manifold $\mathcal{M}_0 \subset \mathbb{R}^D$ of dimension $d < D$. This induces a time-dependent manifold $\mathcal{M}_{t}=\{(1-t)y:y\in\mathcal{M}_{0}\}$. We analyze the process within a \textbf{tubular neighborhood} $\mathcal{T}_\tau = \{x \in \mathbb{R}^D : d(x, \mathcal{M}_t) < \tau\}$.
Inside $\mathcal{T}_\tau$, any point $x$ admits a unique orthogonal decomposition $x = \pi(x) + n(x)$, where $\pi(x) \in \mathcal{M}_t$ is the projection and $n(x)$ is the normal vector.

We denote the orthogonal projection onto the tangent space $T_y \mathcal{M}$ by $P_T(y)$ and onto the normal space by $P_N(y) = I - P_T(y)$. The \textit{intrinsic gradient} is defined as $\nabla_T f(y) := P_T(y) \nabla f(y)$. Finally, we let $\Pi_y(\cdot, \cdot)$ denote the \textit{Second Fundamental Form}, which characterizes the local curvature. Its trace is the \textit{Mean Curvature Vector} $H(y) = \text{tr}(\Pi_y) \in N_y\mathcal{M}$.

\section{Geometric Framework: Score-Guided Proximal Projection}
\subsection{The Proximal Objective}
We formulate the recovery of a clean image from a reference $\xref$ (which may be noisy or OOD) as a proximal optimization problem. We define the likelihood of the reference image as $p(\xref \mid x_0) = \N(\xref \mid x_0, \sigma_p^2)$, where $\sigma_p^2$ is the proximal variance hyperparameter. We define the time-dependent energy potential:
\begin{equation} \label{eq:sgpp-potential}
    J_t(x_t) = \underbrace{\frac{1}{2\sigma_p^2(t)} \|x_t - (1-t)\xref\|^2}_{\text{Fidelity Potential}} - \underbrace{\log p_t(x_t)}_{\text{Generative Potential}}
\end{equation}
Here, $\sigma_p(t) = \sqrt{(1-t)^2\sigma_p^2 + t^2}$ represents the combined variance of the proximal relaxation and the intrinsic flow noise. Minimizing this objective via gradient descent yields the core update rule of our framework:

\begin{equation}
    \label{eq:sgpp_update}
    x_{k+1} = x_k + \eta_k \left( s_\psi(x_k,t_k) - \frac{x_k - (1-t_k)\xref}{(1-t_k)^2\sigma_p^2 + t_k^2} \right)
\end{equation}
where, $s_\psi(x, t) \approx \nabla \log p_t(x)$ is the pre-trained score function obtained from rectified flow models like FLUX \citep{labs2025flux1kontextflowmatching}, $x_0:= x_{t_0} =(1-t_0)\xref + t_0 Z$, with $Z\sim \N(0, I)$ and the time-schedule $1\geq t_0 > t_1> \dots> t_N \geq 0$ forms a homotopy from coarse to fine manifolds. To analyze the dynamics of this update, we first decompose the score field into its intrinsic geometric components relative to the time-dependent manifold $\mathcal{M}_t$.


\begin{assumption}[Tubular Curvature Bound]
\label{ass:tubular_bound}
Let $\kappa_{\max} = \sup_{y \in \mathcal{M}_0} \|\Pi(y)\|_{\text{op}}$ be the maximum principal curvature of the manifold. We assume the data support and diffusion trajectories are confined to a tubular neighborhood $\mathcal{T}_{\tau} = \{x : d(x, \mathcal{M}_0) < \tau\}$ where the radius $\tau$ satisfies the strict bound
    $\tau \cdot \kappa_{\max} \le 1 - \delta$
for some fixed constant $\delta \in (0, 1)$.
\end{assumption}
\begin{proposition}\label{prop:rf-score-decomp}
    For any $x_t$ near $\mathcal{M}_t$, let $n_t = x_t - \pi_{t}$, where $\pi_t:=\pi_{\M_t}(x_t)$ is the nearest point to $x_t$ on the manifold $\M_t$. Then under Assumption \ref{ass:tubular_bound}, the RF score decomposes as:
    \begin{equation*}
        \nabla_{x_t} \log p_t(x_t) = -\frac{n_t}{t^2} + \nabla_T \log p_{\mathcal{M}_t}(\pi_t) + \frac{1}{2}{H}_t + O(1).
    \end{equation*}
    where $H_t \in N_{\pi_t}\M_t$ is the mean curvature vector at $\pi_t$.
\end{proposition}

Substituting this decomposition into the SGPP update rule reveals a critical stability property: the flow naturally suppresses deviations from the manifold.

\begin{proposition}[Normal Contraction]
Let $n_k = x_k - \pi_{\M_{t_k}}(x_k)$ be the normal displacement at step $k$. Under Assumption \ref{ass:tubular_bound}, if the step size $\eta_k$ satisfies the stability condition:$$0 < \eta_k < 2/\br{{1}/{t_k^2} + {1}/{\sigma_p^2(t_k)}} \approx 2 t_k^2 \quad (\text{as } t_k \to 0),$$
then the magnitude of the normal displacement strictly contracts according to:$$\|n_{k+1}\| \le (1 - \lambda_k) \|n_k\| + \eta_k C_N,$$
where the Contraction Rate $\lambda_k = \eta_k \left( \frac{1}{t_k^2} + \frac{1}{\sigma_p^2(t_k)} \right) \in (0, 2)$ dictates the speed of convergence to the manifold. The forcing term $C_N< \infty$ is strictly bounded due to the compactness of $\mathcal{M}_0$ and the smoothness of the reference projection.
\end{proposition}
With the trajectory guaranteed to stay near the manifold, we now characterize its motion along the surface, which corresponds to the semantic evolution of the image.

\begin{proposition}[Tangential Drift]
Let $\pi_k$ be the projection of the sampling trajectory $x_k$ onto the manifold $\mathcal{M}_{t_k}$. The evolution of $\pi_k$ follows the intrinsic gradient flow perturbed by a curvature-induced drift:
\begin{equation*}
    \pi_{k+1} - \pi_k = \eta_k \underbrace{\left( \nabla_T \log p_{\mathcal{M}_{t_k}}(\pi_k) - \frac{1}{\sigma_p^2(t_k)} P_{T_k}(\pi_k - (1-t_k)\xref) \right)}_{\text{Ideal Semantic Velocity } v_\textrm{tan}} + \eta_k \mathcal{E}_{drift} + O(\eta_k^2),
\end{equation*}
where $P_{T_k}$ is the orthogonal projection onto the tangent space $T_{\pi_k}\mathcal{M}_{t_k}$, and the geometric drift error is bounded by:
\begin{equation*}
    \|\mathcal{E}_{drift}\| \le \frac{\kappa_{max} \|n_k\|}{1 - \kappa_{max}\|n_k\|} \|v_\textrm{tan}\|.
\end{equation*}
\end{proposition}
Finally, we characterize the fixed point of this dynamical system, establishing the connection between our geometric flow and Bayesian inference.
\begin{theorem}[Fixed-Point Characterization \& MAP Equivalence]
\label{thm:fixed_point}
Let $x^* = \pi^* + n^*$ be the instantaneous fixed point of the SGPP update rule at a fixed time $t \in (0, 1)$. Under Assumption 3.1, the equilibrium state is characterized by two decoupled conditions:

\begin{enumerate}
    \item \textbf{Normal Equilibrium (manifold safety):} The normal displacement $n^*$ is non-zero but suppressed quadratically by the time parameter:
    \begin{equation*}
        \|n^*\| \le \frac{t^2 \sigma_p^2(t)}{t^2 + \sigma_p^2(t)} C_N(t) \approx t^2 C_N(t) \quad (\text{as } t \to 0),
    \end{equation*}
    where $C_N(t)$ is the bounded forcing term defined in Proposition 3.3.

    \item \textbf{Tangential Equilibrium (MAP equivalence):} The projection $\pi^*$ satisfies the stationarity condition on the manifold surface:
    \begin{equation*}
        \nabla_T \log p_{\mathcal{M}_t}(\pi^*) = \frac{1}{\sigma_p^2(t)} P_{T_{\pi^*}}(\pi^* - (1-t)\xref).
    \end{equation*}
    In the limit as $t \to 0$, $\pi^*$ converges to the \textbf{manifold-constrained MAP estimator}:
    \begin{equation*}
        \pi^* \to \arg\max_{y \in \mathcal{M}_0} \left[ \log p_{\M_0}(y) - \frac{1}{2\sigma_p^2} \|y - \xref\|^2 \right].
    \end{equation*}
\end{enumerate}
\end{theorem}
Theorem \ref{thm:fixed_point} proves that the equilibrium of our dynamic system corresponds exactly to the Manifold-Constrained MAP estimator. However, unlike prior methods like MCG which require an explicit, computationally expensive projection operator $P_{\mathcal{M}}(x)$, SGPP implements this constraint implicitly. The pre-trained score field itself acts as the projection operator. 
\section{From Optimization to Sampling}
While Theorem \ref{thm:fixed_point} ensures geometric safety, the proximal update converges to the posterior \textit{mode} (MAP). In high dimensions, the mode often holds little probability mass and fails to represent the typical set, leading to over-smoothed results. To recover high-frequency textures and diversity, we must sample from the posterior distribution rather than optimize for its peak.

The theoretical posterior probability path is governed by the ODE:
\begin{equation} \label{eq:rf-ode}
    \frac{dx_t}{dt} = v_t(x_t\mid \xref) = -\frac{x_t}{1-t} - \frac{t}{1-t} \left( \nabla \log p_t(x_t) + \nabla \log p_t(\xref \mid x_t) \right), \quad t: 1 \to 0.
\end{equation}
Using the Gaussian likelihood $p_t(\xref \mid x_t) \propto \exp\left( -\|x_t - (1-t)\xref\|^2 / 2\sigma^2_p(t) \right)$, this formulation directly links our energy potential $J_t$ in (\ref{eq:sgpp-potential}) to the negative log-posterior.
While we prove in Appendix \ref{app:rf-ode-sampler} that this ODE targets the exact posterior, in practice it tends to collapse to the MAP as the guidance force overpowers the ``memory'' of the initial noise $X_1$. To preserve diversity, we adopt the stochastic sampler derived by \citet[Lemma A.4]{rout2024semanticimageinversionediting}:
\begin{equation} \label{eq:prf_sde}
    dX_t =  \left(-\frac{x_t}{1-t} - \frac{2t}{1-t} \nabla \log p_t(x_t \mid \xref)  \right) dt + \sqrt{\frac{2t}{1-t}} dB_t, \quad t: 1 \to 0,
\end{equation}
where $\{B_t\}$ is standard Brownian motion.

\section{Connection to RF-Inversion}

We reinterpret the control field in RF-Inversion \citep{rout2024semanticimageinversionediting} as sampling from a \textit{geometric mixture} distribution. Their control update corresponds to a convex combination of the generative and guidance scores:
\begin{equation} \label{eq:mixture_score}
    \nabla \log p^\eta_t(x_t \mid \xref) := (1-\eta) \nabla \log p_t(x_t) + \eta \nabla \log p_t(\xref \mid x_t)
\end{equation}
in the limit of $\sigma_p \to 0$ (see Appendix \ref{app:rf-inversion} for the proof).

This formulation highlights a continuum of control. As $\eta \to 1$, the prior vanishes, locking the trajectory to the reference path $\xref$ (hard guidance). While intermediate $\eta$ values allow for "soft guidance" they regulate only the \textit{direction} of the update, not the \textit{validity} of the constraint. In standard RF-Inversion (where implicit $\sigma_p \to 0$), high $\eta$ creates adversarial gradients: the text score pushes for change while the reference score enforces strict stasis, causing artifacts.

Our proximal variance $\sigma_p$ resolves this by introducing \textit{geometric tolerance}. It effectively widens the manifold tube, making the constraint ``elastic.'' This allows the semantic drive ($\eta$) to satisfy the prompt without violating the structural essence of the reference image.

\bibliography{iclr2026_delta}

@misc{chung2024improvingdiffusionmodelsinverse,
      title={Improving Diffusion Models for Inverse Problems using Manifold Constraints}, 
      author={Hyungjin Chung and Byeongsu Sim and Dohoon Ryu and Jong Chul Ye},
      year={2024},
      eprint={2206.00941},
      archivePrefix={arXiv},
      primaryClass={cs.LG},
      url={https://arxiv.org/abs/2206.00941}, 
}

@misc{chung2024diffusionposteriorsamplinggeneral,
      title={Diffusion Posterior Sampling for General Noisy Inverse Problems}, 
      author={Hyungjin Chung and Jeongsol Kim and Michael T. Mccann and Marc L. Klasky and Jong Chul Ye},
      year={2024},
      eprint={2209.14687},
      archivePrefix={arXiv},
      primaryClass={stat.ML},
      url={https://arxiv.org/abs/2209.14687}, 
}

@misc{rout2024semanticimageinversionediting,
      title={Semantic Image Inversion and Editing using Rectified Stochastic Differential Equations}, 
      author={Litu Rout and Yujia Chen and Nataniel Ruiz and Constantine Caramanis and Sanjay Shakkottai and Wen-Sheng Chu},
      year={2024},
      eprint={2410.10792},
      archivePrefix={arXiv},
      primaryClass={cs.LG},
      url={https://arxiv.org/abs/2410.10792}, 
}

@misc{liu2022flowstraightfastlearning,
      title={Flow Straight and Fast: Learning to Generate and Transfer Data with Rectified Flow}, 
      author={Xingchao Liu and Chengyue Gong and Qiang Liu},
      year={2022},
      eprint={2209.03003},
      archivePrefix={arXiv},
      primaryClass={cs.LG},
      url={https://arxiv.org/abs/2209.03003}, 
}

@misc{kulikov2025floweditinversionfreetextbasedediting,
      title={FlowEdit: Inversion-Free Text-Based Editing Using Pre-Trained Flow Models}, 
      author={Vladimir Kulikov and Matan Kleiner and Inbar Huberman-Spiegelglas and Tomer Michaeli},
      year={2025},
      eprint={2412.08629},
      archivePrefix={arXiv},
      primaryClass={cs.CV},
      url={https://arxiv.org/abs/2412.08629}, 
}

@misc{kim2025flowaligntrajectoryregularizedinversionfreeflowbased,
      title={FlowAlign: Trajectory-Regularized, Inversion-Free Flow-based Image Editing}, 
      author={Jeongsol Kim and Yeobin Hong and Jonghyun Park and Jong Chul Ye},
      year={2025},
      eprint={2505.23145},
      archivePrefix={arXiv},
      primaryClass={cs.CV},
      url={https://arxiv.org/abs/2505.23145}, 
}

@misc{liu2026improvingeuclideandiffusiongeneration,
      title={Improving the Euclidean Diffusion Generation of Manifold Data by Mitigating Score Function Singularity}, 
      author={Zichen Liu and Wei Zhang and Tiejun Li},
      year={2026},
      eprint={2505.09922},
      archivePrefix={arXiv},
      primaryClass={cs.LG},
      url={https://arxiv.org/abs/2505.09922}, 
}

@misc{labs2025flux1kontextflowmatching,
      title={FLUX.1 Kontext: Flow Matching for In-Context Image Generation and Editing in Latent Space},
      author={Black Forest Labs and Stephen Batifol and Andreas Blattmann and Frederic Boesel and Saksham Consul and Cyril Diagne and Tim Dockhorn and Jack English and Zion English and Patrick Esser and Sumith Kulal and Kyle Lacey and Yam Levi and Cheng Li and Dominik Lorenz and Jonas Müller and Dustin Podell and Robin Rombach and Harry Saini and Axel Sauer and Luke Smith},
      year={2025},
      eprint={2506.15742},
      archivePrefix={arXiv},
      primaryClass={cs.GR},
      url={https://arxiv.org/abs/2506.15742},
}

@misc{song2021scorebasedgenerativemodelingstochastic,
      title={Score-Based Generative Modeling through Stochastic Differential Equations}, 
      author={Yang Song and Jascha Sohl-Dickstein and Diederik P. Kingma and Abhishek Kumar and Stefano Ermon and Ben Poole},
      year={2021},
      eprint={2011.13456},
      archivePrefix={arXiv},
      primaryClass={cs.LG},
      url={https://arxiv.org/abs/2011.13456}, 
}
\bibliographystyle{iclr2026_delta}

\appendix

\section{Related Works}

\textbf{Rectified Flows and Inversion-Based Editing.}
Rectified Flows (RF) \citep{liu2022flowstraightfastlearning} straighten the transport paths between data and noise, enabling faster and more stable sampling compared to standard diffusion models. Recent advances have extended RFs to controlled generation via inversion. Notably, \cite{rout2024semanticimageinversionediting} introduced \textit{RF-inversion} for semantic editing, deriving a control term that forces generation to retrace the inverted noise path. Our work unifies this perspective: we demonstrate that the control mechanism by \cite{rout2024semanticimageinversionediting} is mathematically equivalent to the \textit{``hard guidance''} limit ($\sigma_p \to 0$) of our proximal framework. By relaxing this to ``soft guidance'' ($\sigma_p > 0$), SGPP introduces a geometric flexibility that allows the model to correct artifacts that rigid inversion would otherwise preserve.

\textbf{Inversion-Free Editing (FlowEdit vs. SGPP).}
A parallel class of methods bypasses inversion entirely by manipulating vector fields directly. {FlowEdit} \citep{kulikov2025floweditinversionfreetextbasedediting} and {FlowAlign} \citep{kim2025flowaligntrajectoryregularizedinversionfreeflowbased} construct a direct ODE mapping by applying the difference between target and source velocity fields ($v_{tgt} - v_{src}$). However, these methods rely on \textit{semantic subtraction}: they require a user-provided \textit{source prompt} to define the structure to be preserved. In contrast, SGPP enforces structure via \textit{geometric projection}: we utilize the reference image itself as a proximal constraint. This renders SGPP a \textit{zero-shot} method, capable of preserving identity and high-frequency details without requiring auxiliary text descriptions or prompt tuning.

\textbf{Diffusion Posterior Sampling (DPS) and Jacobian-Free Guidance.}
Solving inverse problems with pre-trained models is a cornerstone of training-free generative AI. {DPS} \citep{chung2024diffusionposteriorsamplinggeneral} approximates the posterior score via a likelihood gradient $\nabla_{x_t} \log p(y|x_t)$, typically formulated as the terminal data consistency $\| y - \hat{x}_0(x_t) \|^2$. While effective for general operators, DPS requires backpropagating through the denoising network to compute the Jacobian $\nabla_{x_t} \hat{x}_0$, which is computationally expensive and notoriously unstable at high noise levels.
Our approach differs fundamentally in objective and stability. Instead of terminal consistency, SGPP optimizes \textbf{trajectory consistency} $\| x_t - (1-t)\xref \|^2$. Leveraging the linear geometry of Rectified Flows, we derive a closed-form, \textit{Jacobian-free guidance term} (Eq. 3). This eliminates the need for backward passes through the network, avoiding the ``exploding gradient'' problem of DPS while providing a rigorous geometric guarantee of \textbf{Normal Contraction} (Prop 3.3) that heuristic approximations lack.

\paragraph{Connection to Manifold Constraints.}
Our framework shares conceptual roots with \textit{Manifold Constrained Gradients (MCG)} \cite{chung2024improvingdiffusionmodelsinverse}, which seeks the Maximum A Posteriori (MAP) estimate constrained to the data manifold. However, where MCG relies on explicit—and often unstable—projections onto an approximated tangent space, SGPP provides an \textit{implicit} implementation. By exploiting the natural decomposition of the Rectified Flow score (Prop 3.2), SGPP achieves the equivalent ``Manifold-Constrained MAP'' solution (Theorem 3.5) purely through the proximal update, ensuring geometric stability without the need for manual projection steps or manifold approximations.

\section{Experiments}
\label{sec:experiments}

We validate Score-Guided Proximal Projection (SGPP) in two regimes: (1) a controlled 2D geometric experiment to verify the Normal Contraction property and stability against baselines, and (2) high-resolution semantic editing and recovery using the state-of-the-art FLUX Rectified Flow model.

\subsection{Geometric Validation: The Normal Contraction Property}

To rigorously test the geometric stability of our method, we evaluate it on a 2D ``two-moons'' manifold distribution. We compare SGPP against two established baselines: Diffusion Posterior Sampling (DPS)and Inversion-Based Editing (RF-Inversion) .

\textbf{Baselines vs. SGPP.} Existing methods struggle with the geometry of the manifold in distinct ways:
\begin{itemize}
    \item \textbf{DPS Instability:} DPS approximates the likelihood score via backpropagation through the denoising network. As seen in our comparisons, it is highly sensitive to the noise scale $\sigma$. At high noise levels, the gradients often explode or misguide the trajectory, leading to points that overshoot the manifold or fail to converge.
    \item \textbf{Geometric Locking in RF-Inversion:} Standard inversion demonstrates the ``geometric locking'' phenomenon. We observe that unless the guidance is stopped very early ($t=0.1$), the trajectories collapse entirely to the reference input $\xref$. This confirms that strict inversion leaves insufficient degrees of freedom for generative corrections.
\end{itemize}

\textbf{SGPP Behavior.} In contrast, SGPP exhibits robust convergence. The deterministic update (\ref{eq:sgpp_update}) effectively ``snaps'' out-of-distribution points onto the manifold spine due to the restoring force derived in Proposition 3.3. The stochastic variant (SGPP-SDE) maintains this geometric safety while correctly sampling the posterior distribution, covering the manifold density rather than collapsing to a single mode.

\subsection{Zero-Shot Semantic Editing}

We apply SGPP to zero-shot semantic editing, a task requiring a delicate balance between preserving the source identity and hallucinating new semantic structures (e.g., transforming a ``cat'' into a ``lion''). We use the FLUX model with a target prompt $y_{tgt} =$ ``a lion''.

\textbf{Overcoming Geometric Locking.} We compare our method against standard RF-Inversion. The baseline fails to generate meaningful semantic changes; the strong geometric constraint forces the output to retain the exact shape and contours of the house cat, resulting in a texture-swapped hybrid rather than a true lion.

\textbf{Soft Guidance.} SGPP overcomes this by relaxing the proximal constraint. By setting $\sigma_p = 0.2$ and using a geometric score mixture ($\eta=0.8$), our method permits the trajectory to deviate tangentially from the reference path. This ``Soft Guidance'' allows the model to hallucinate the necessary structural changes (e.g., the lion's mane and broader muzzle) while the proximal term ensures the pose and background remain consistent with the reference. Notably, this is achieved without any inversion step or auxiliary control networks.

\subsection{The Fidelity-Realism Trade-off}

Finally, we analyze the effect of the proximal variance $\sigma_p$ on the reconstruction quality. This hyperparameter acts as a direct ``knob'' for the perception-distortion trade-off:

\begin{itemize}
    \item \textbf{Tight Constraint ($\sigma_p \to 0$):} At low variance (e.g., $\sigma_p = 0.01$), the method recovers the reference image almost exactly, correcting only minor artifacts. This corresponds to the ``Hard Guidance'' limit.
    \item \textbf{Generative Freedom ($\sigma_p > 0$):} As we increase $\sigma_p$ (e.g., to $0.2$ or $0.5$), the model gains the freedom to hallucinate high-frequency details that are statistically likely under the prior but absent in the reference.
\end{itemize}

This confirms that SGPP provides a continuous, controllable spectrum between strict fidelity (reconstruction) and unconstrained generation (realism), validating our theoretical unification of optimization and sampling.
\begin{figure}[ht]
    \centering
    
    \begin{subfigure}{\textwidth}
        \centering
        \includegraphics[width=\linewidth]{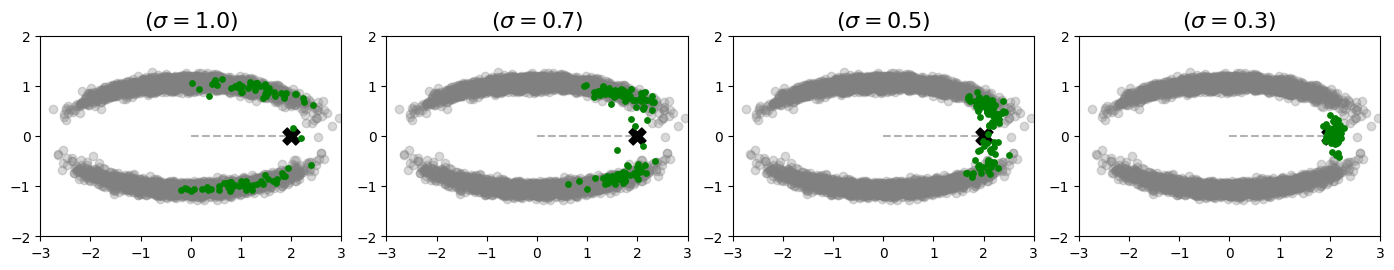}
        \caption{Samples generated by DPS for different $\sigma$ values with the observation model $\xref = x_0 + \sigma Z$, where $Z \sim \N(0, I)$}
    \end{subfigure}

    \begin{subfigure}{\textwidth}
        \centering
        \includegraphics[width=\linewidth]{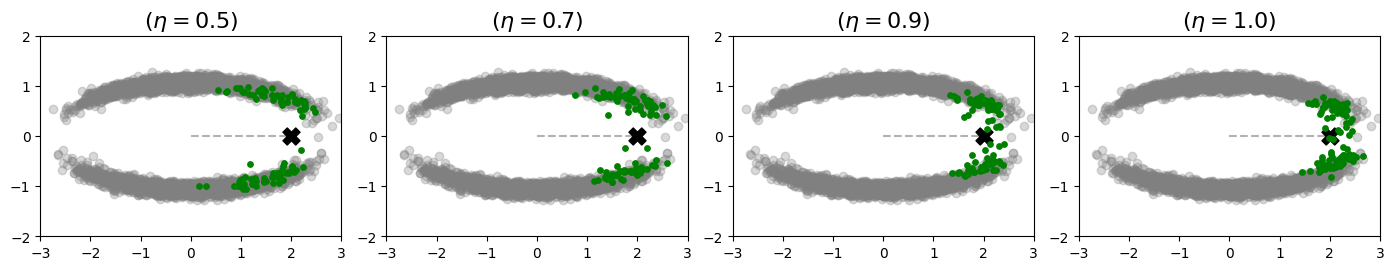}
        \caption{Samples generated by RF-Inversion with inversion guidance $\gamma=0.5$ and editing guidance $\eta = 0.8$. We stop the guidance at $t=0.1$, otherwise all samples collapse to $\xref$}
    \end{subfigure}
    
    \begin{subfigure}{\textwidth}
        \centering
        \includegraphics[width=\linewidth]{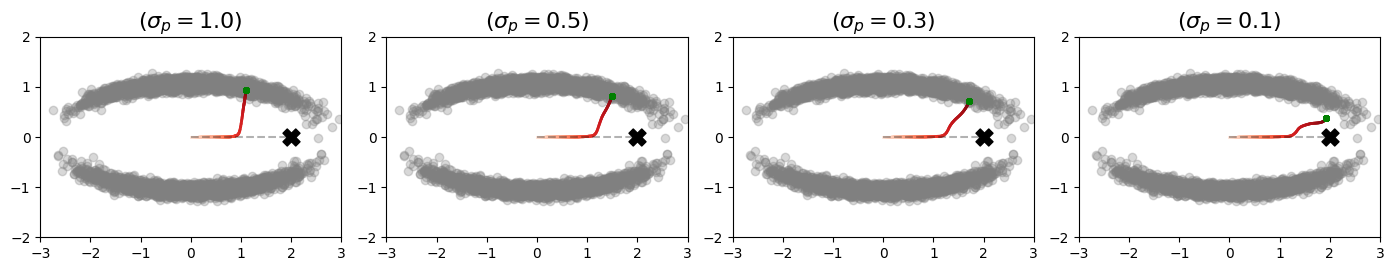}
        \caption{\textbf{(Ours)} Trajectories of SGPP (the proximal gradient ascent update in (\ref{eq:sgpp_update})) leading to the posterior modes for varying values of proximal variance}
    \end{subfigure}

    \begin{subfigure}{\textwidth}
        \centering
        \includegraphics[width=\linewidth]{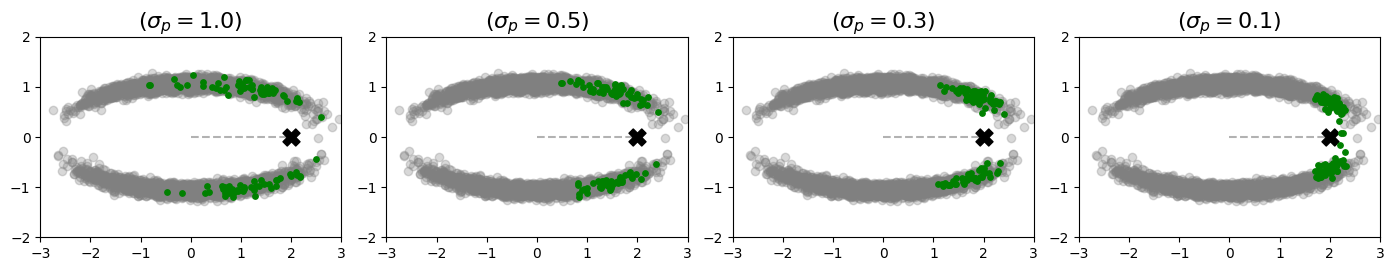}
        \caption{\textbf{(Ours)} Samples generated by the SGPP-SDE (\ref{eq:prf_sde}) for varying values of proximal variance}
    \end{subfigure}

    \caption{Camparison of different methods (a) DPS (b) RF-inversion (c) SGPP (deterministic optimizer) (d) SGPP-SDE}
    \label{fig:four_images_stacked}
\end{figure}

\begin{figure}[ht]
    \centering
    \begin{subfigure}[b]{0.3\textwidth}
        \centering
        \includegraphics[width=\linewidth]{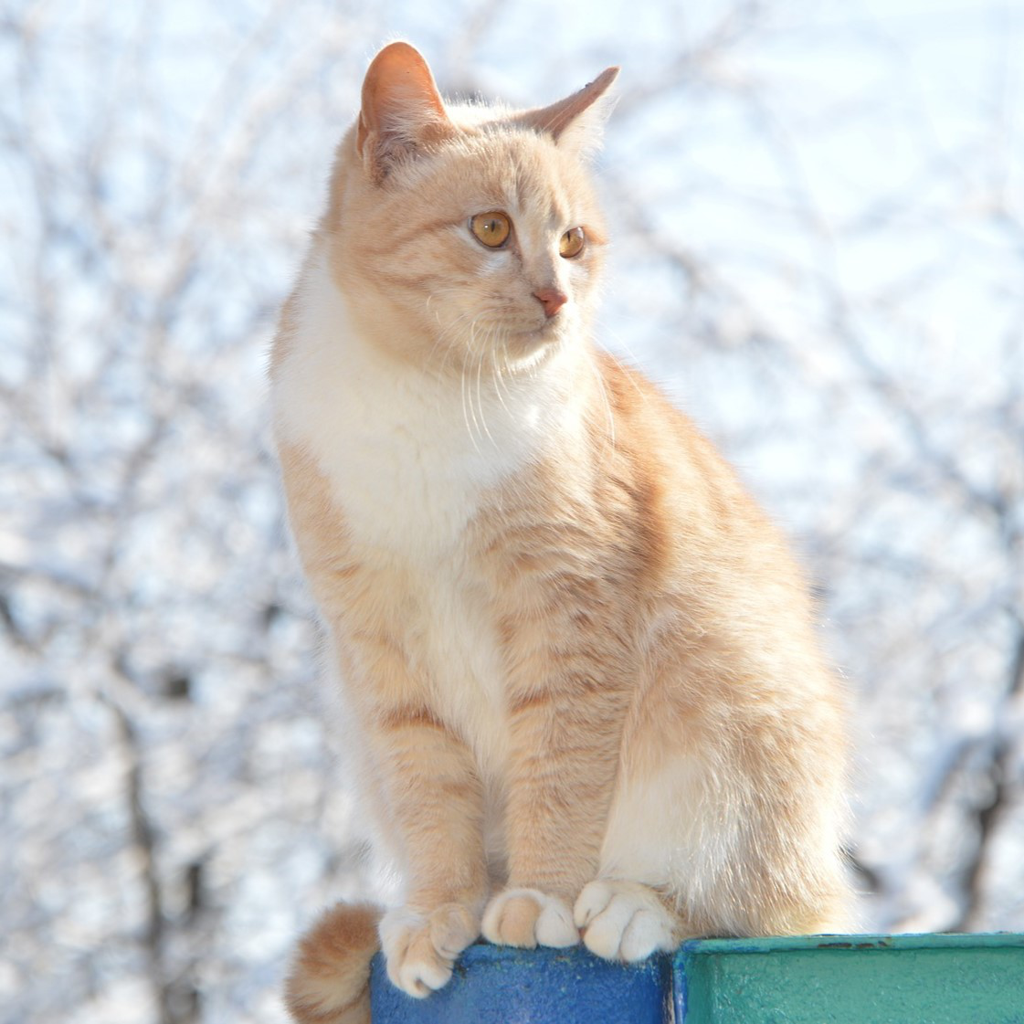}
        \caption{Original cat image}
    \end{subfigure}
    \hfill 
    \begin{subfigure}[b]{0.3\textwidth}
        \centering
        \includegraphics[width=\linewidth]{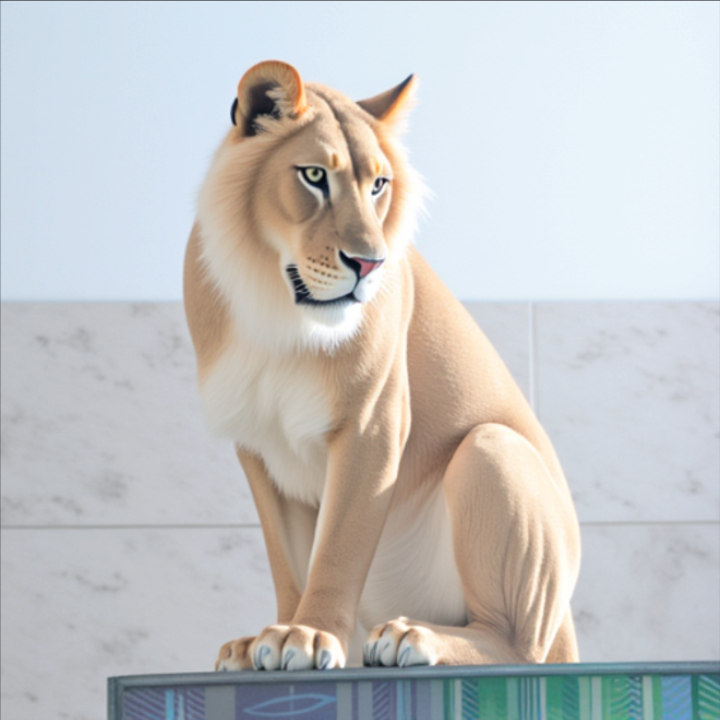}
        \caption{RF-inversion}
    \end{subfigure}
    \hfill 
    \begin{subfigure}[b]{0.3\textwidth}
        \centering
        \includegraphics[width=\linewidth]{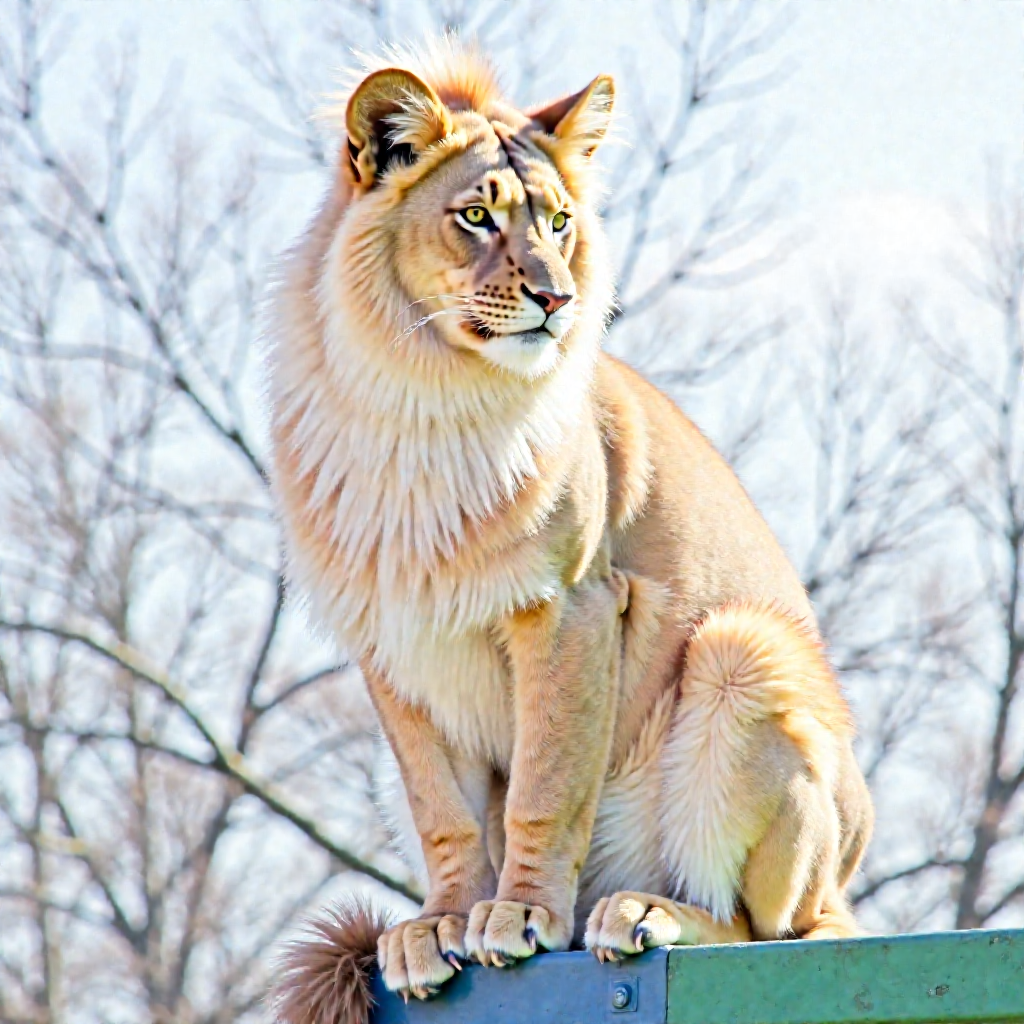}
        \caption{\textbf{(Ours) Proximal RF-SDE}}
    \end{subfigure}

    \caption{\textbf{Semantic Editing}. We use the FLUX model with the target prompt ``a lion". (b) We extract the RF-inversion edited image from their standard implementation. (c) We use our proximal SGPP-SDE with the geometric mixture score at $\eta=0.8$ and $\sigma_p = 0.2$ and CFG guidance 3.5. We do \textbf{not} use inversion and stop the proximal guidance at $t_{stop}=15/28$.}
\end{figure}

\begin{figure}[ht]
    \centering
    \begin{subfigure}[b]{0.22\textwidth}
        \centering
        \includegraphics[width=\linewidth]{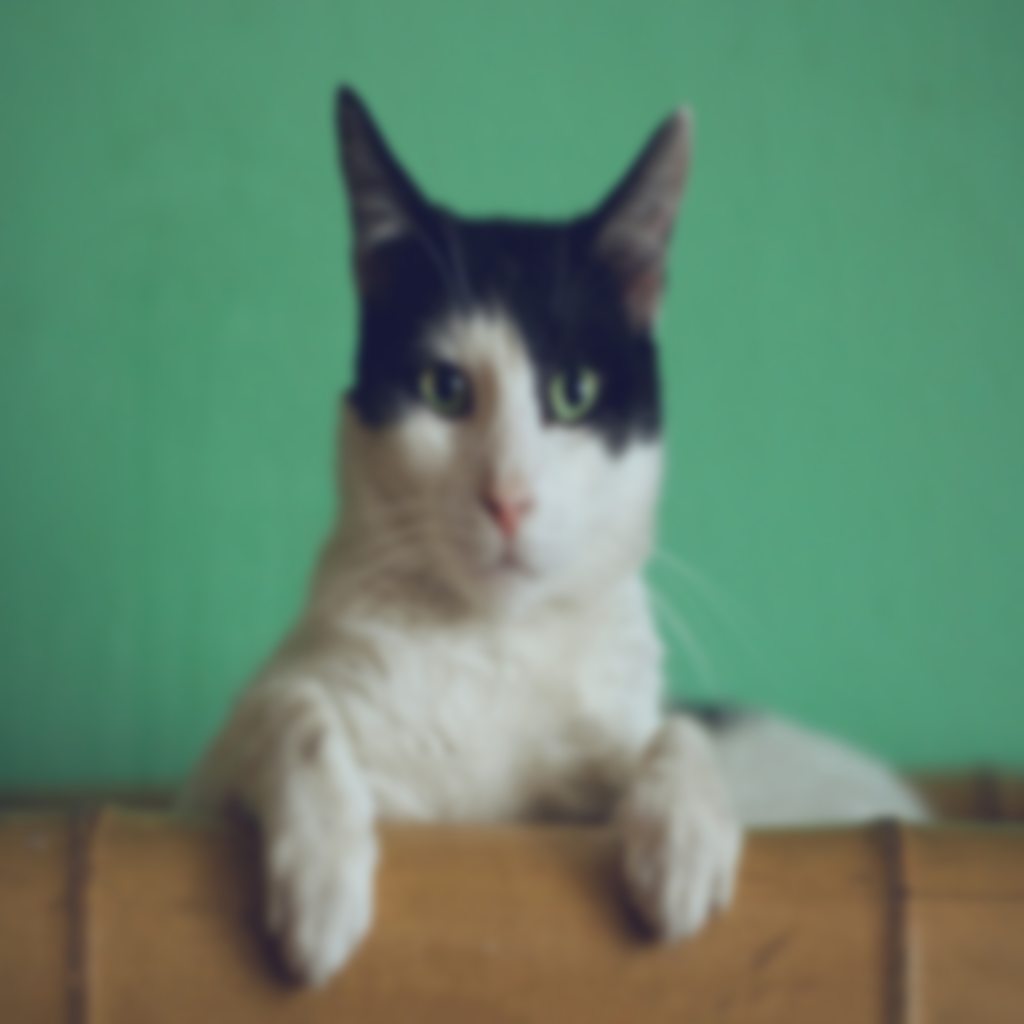}
        \caption{Original cat image}
    \end{subfigure}
    \hfill 
    \begin{subfigure}[b]{0.22\textwidth}
        \centering
        \includegraphics[width=\linewidth]{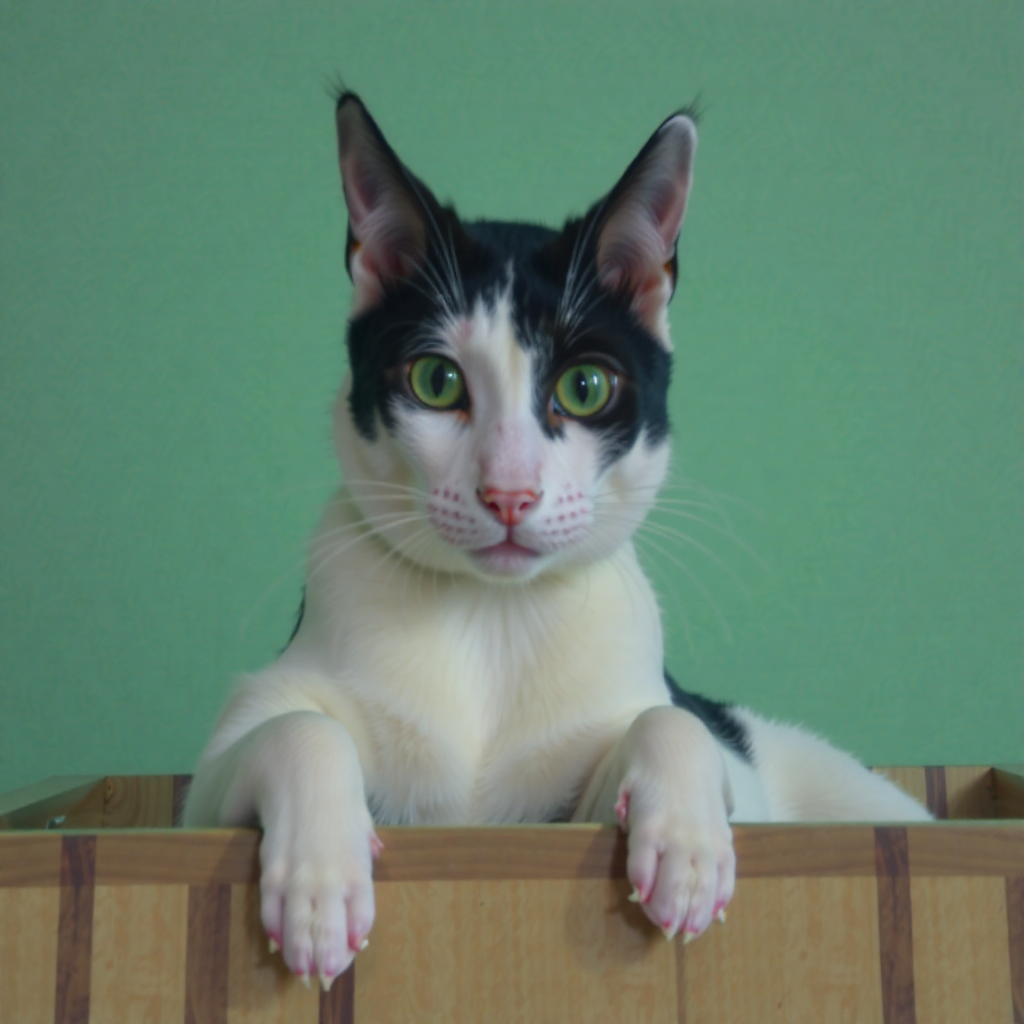}
        \caption{$\sigma_p = 0.5$}
    \end{subfigure}
    \hfill 
    \begin{subfigure}[b]{0.22\textwidth}
        \centering
        \includegraphics[width=\linewidth]{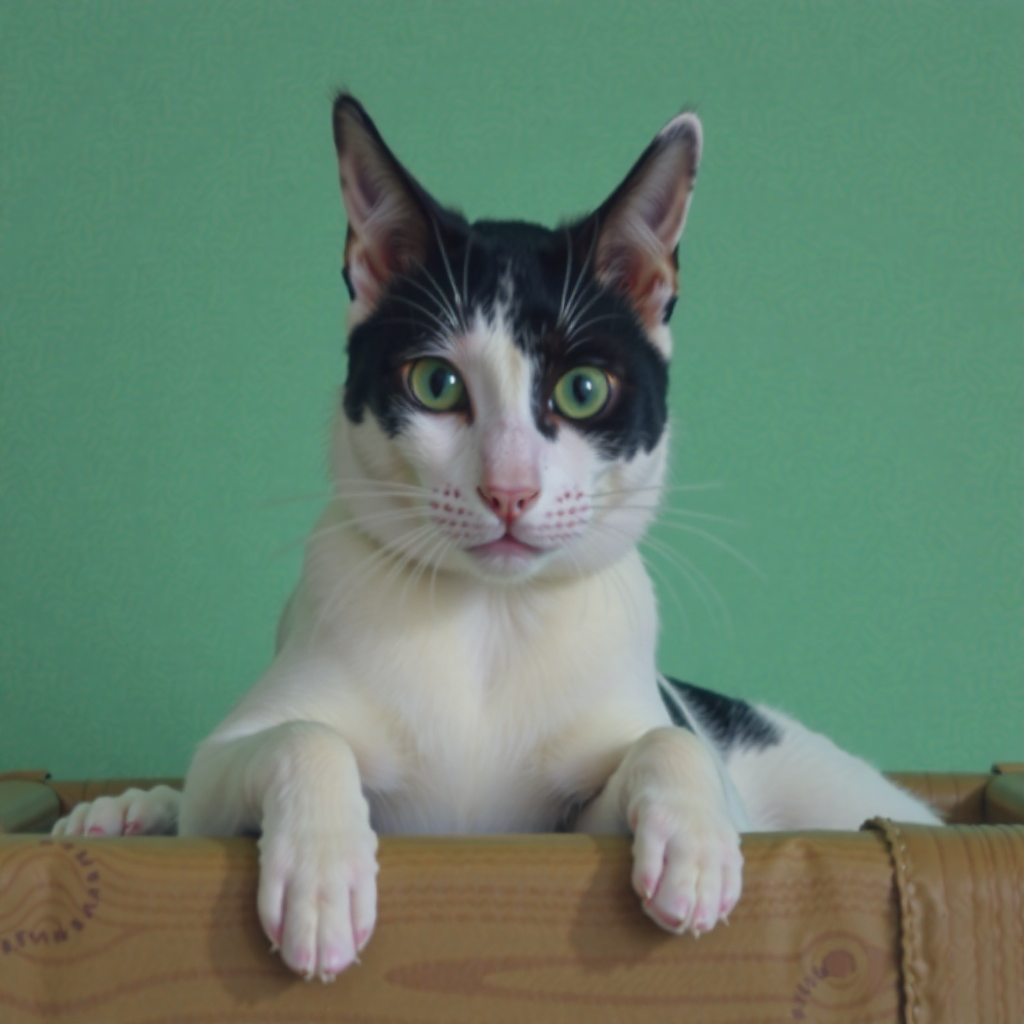}
        \caption{$\sigma_p = 0.2$}
    \end{subfigure}
    \hfill
    \begin{subfigure}[b]{0.22\textwidth}
        \centering
        \includegraphics[width=\linewidth]{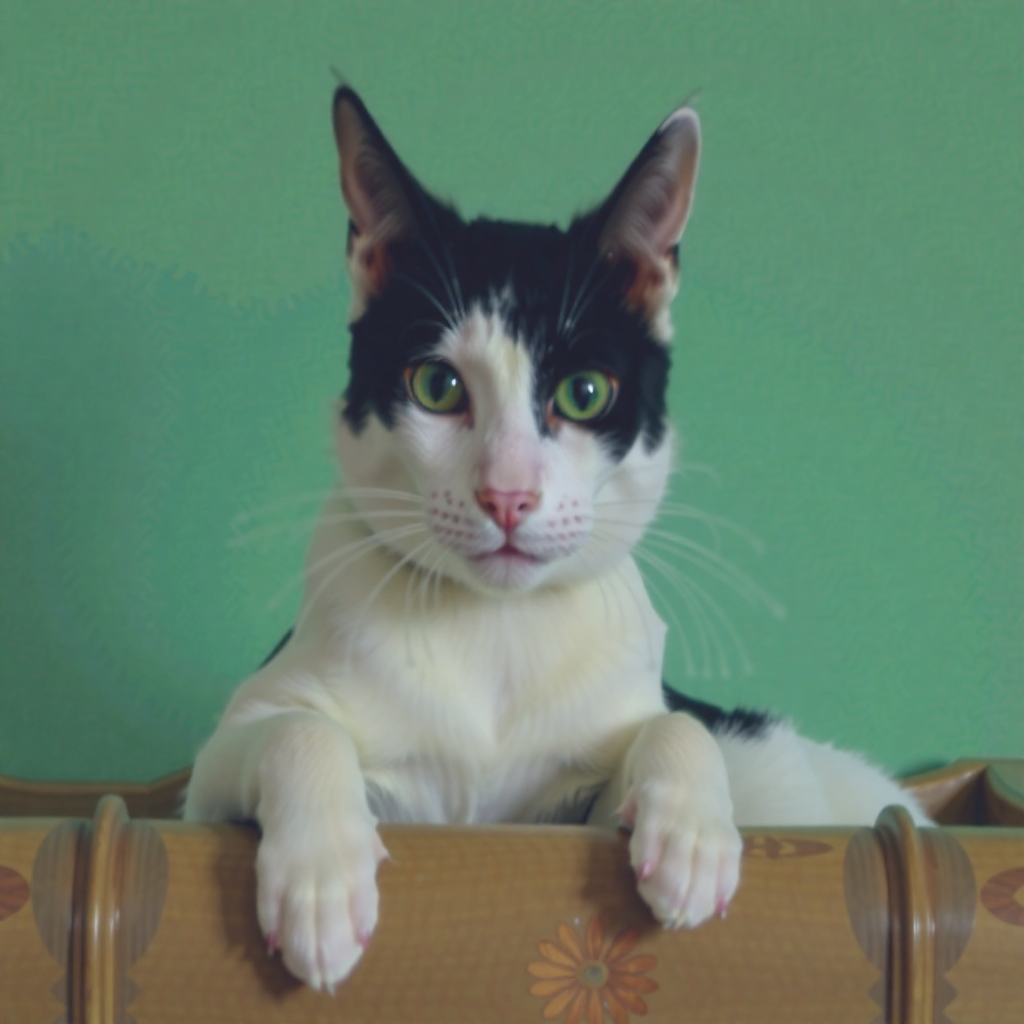}
        \caption{$\sigma_p = 0.01$}
    \end{subfigure}

    \caption{\textbf{Reconstruction.} We use the FLUX model with no target prompt.  We use our proximal SGPP-SDE with the the posterior score as in (\ref{eq:prf_sde}). We do \textbf{not} use inversion and stop the proximal guidance at $t_{stop}=0.6$.}
\end{figure}

\section{Background on Local geometry}

\subsection{Local Quadratic Approximation}
To characterize the curvature, we view the manifold locally as a quadratic surface. For any point $y \in \mathcal{M}$, we can represent nearby points on the manifold using a local coordinate chart. Let $v$ be a small vector in the tangent space $T_y \mathcal{M}$. The position of a point on the manifold, denoted $\phi(v)$, can be approximated via a Taylor expansion:
\begin{equation*}
    \phi(v) = y + v + \frac{1}{2}\Pi_y(v, v) + O(\|v\|^3).
\end{equation*}
Here, $\Pi_y: T_y \mathcal{M} \times T_y \mathcal{M} \to N_y\mathcal{M} $ is the \textit{Second Fundamental Form}. Intuitively, it represents the quadratic ``bending" of the manifold away from its tangent plane. It takes a tangent direction $v$ and outputs a normal vector describing how the surface curves in that direction.

We define the \textit{Mean Curvature Vector} $H(y)$ as the trace of this quadratic term:
\begin{equation*}
    H(y) := \sum_{i=1}^d \Pi_y(e_i, e_i),
\end{equation*}
where $\{e_i\}_{i=1}^d$ is any orthonormal basis of the tangent space. Physically, $H(y)$ points in the average direction of the manifold's curvature (e.g., towards the center of a sphere).

\subsection{Curvature Interaction}
To analyze how the geometry affects the diffusion score, we must quantify how the curvature interacts with the normal distance. For a fixed normal vector $n$, we define the \textit{Shape Operator} $\mathcal{S}_n$ as the linear map on the tangent space that satisfies:
\begin{equation*}
    \langle v, \mathcal{S}_n v \rangle = \langle n, \Pi_y(v, v) \rangle.
\end{equation*}
This operator scales tangent vectors based on the principal curvatures in the direction of $n$. The condition that a point $x$ lies within the reach $\tau$ is equivalent to the condition that $\|\mathcal{S}_{n(x)}\|_{\text{op}} < 1$. This ensures that the normal fibers do not cross, preventing the formation of singularities in the projection map.

\begin{proposition}
Let $x \in \mathcal{T}_\tau$ be a point within the valid tube of the manifold $\mathcal{M}_0$. Consider the marginal density $p_{\sigma}(x)$ of the model $X = Y + \sigma Z$, where $Y \sim p_{\mathcal{M}_0}$ and $Z \sim \mathcal{N}(0, I_D)$. As $\sigma \to 0$, the score decomposes as:
\begin{equation*}
    \nabla_x \log p_{\sigma}(x) = -\frac{n}{\sigma^2} + \nabla_T \log p_{\mathcal{M}_0}(\pi(x)) + \frac{1}{2}H(\pi(x)) + O(\sigma + \|n\|)
\end{equation*}
where $n = x - \pi(x)$ is the normal vector, $\nabla_T$ is the intrinsic Riemannian gradient, and $H$ is the mean curvature vector.
\end{proposition}

\begin{proof}
The proof utilizes {Tweedie's Formula}, which relates the score of a Gaussian convolution to the posterior expectation of the clean variable $Y$:
\begin{equation}
    \label{eq:tweedie}
    \nabla_x \log p_{\sigma}(x) = \frac{1}{\sigma^2} \left( \mathbb{E}[Y \mid X=x] - x \right).
\end{equation}
We proceed by computing the conditional expectation $\mathbb{E}[Y \mid x]$ via a Laplace approximation on the tangent space.

\paragraph{Local Geometry Setup.}
Let $\bar{y} = \pi(x)$. We parameterize the manifold locally using the exponential map $\exp_{\bar{y}}: T_{\bar{y}}\mathcal{M}_0 \to \mathcal{M}_0$. A point $y$ near $\bar{y}$ is represented by a tangent vector $v \in \mathbb{R}^d$ such that:
\begin{equation*}
    y(v) = \bar{y} + v + \frac{1}{2}\Pi(v, v) + O(\|v\|^3),
\end{equation*}
where $\Pi$ is the Second Fundamental Form. The observation $x$ decomposes as $x = \bar{y} + n$.

\paragraph{The Posterior Distribution.}
The posterior density $q(v \mid x)$ on the tangent space is proportional to the likelihood times the prior. Note that the Riemannian volume element expansion $dV_y = (1 + O(\|v\|^2)) dV_v$ implies the Jacobian contribution to the energy is quadratic in $v$, and thus does not affect the linear drift term $\nabla_T \log p_{\mathcal{M}_0}$ at leading order. The likelihood energy $E(v) = \frac{1}{2\sigma^2}\|x - y(v)\|^2$ expands as:
\begin{align}
    \|x - y(v)\|^2 &= \|(\bar{y} + n) - (\bar{y} + v + \frac{1}{2}\Pi(v,v))\|^2 + O(\|v\|^3) \\
    &= \|n - v - \frac{1}{2}\Pi(v,v)\|^2 \\
    &= \|n\|^2 + \|v\|^2 - 2\langle n, v \rangle - \langle n, \Pi(v,v) \rangle + \frac{1}{4}\|\Pi\|^2.
\end{align}
Using the orthogonality conditions $\langle n, v \rangle = 0$ and keeping terms up to second order, the relevant energy functional is:
\begin{equation*}
    E(v) = \text{const} + \frac{1}{2\sigma^2} \left( \|v\|^2 - \langle n, \Pi(v,v) \rangle \right) - \langle v, \nabla_T \log p_{\mathcal{M}_0} \rangle + O(\|v\|^3/\sigma^2).
\end{equation*}
This corresponds to a Gaussian distribution $\mathcal{N}(\mu_v, \Sigma_v)$ with precision matrix $\Lambda_v = \frac{1}{\sigma^2}(I - \mathcal{S}_n)$, where $\mathcal{S}_n$ is the shape operator defined by $\langle u, \mathcal{S}_n v \rangle = \langle n, \Pi(u, v) \rangle$.

\paragraph{Stability and Expectation.}
By Assumption \ref{ass:tubular_bound}, $\|n\| \|\Pi\|_{\text{op}} < 1$, ensuring $\|\mathcal{S}_n\|_{\text{op}} < 1$. Thus, $\Lambda_v$ is positive definite. We rigorously approximate the covariance $\Sigma_v = \Lambda_v^{-1}$ using the Neumann series expansion:
\begin{equation*}
    \Sigma_v = \sigma^2 (I - \mathcal{S}_n)^{-1} = \sigma^2 (I + \mathcal{S}_n + O(\|n\|^2)) = \sigma^2 I + O(\sigma^2 \|n\|).
\end{equation*}
Since $\|\mathcal{S}_n\|_{\text{op}} \le \kappa_{\max} \|n\|$, the anisotropic term $\sigma^2 \mathcal{S}_n$ is of order $O(\sigma^2 \|n\|)$. Thus, for the leading order trace calculation, we may approximate $\Sigma_v \approx \sigma^2 I$ with an error of $O(\sigma^2 \|n\|)$.
We now compute the expectation of the embedding $y(v)$:
\begin{equation*}
    \mathbb{E}[y(v) \mid x] = \bar{y} + \mathbb{E}[v \mid x] + \frac{1}{2}\mathbb{E}[\Pi(v, v) \mid x] + \mathbb{E}[O(\|v\|^3) \mid x].
\end{equation*}

\begin{enumerate}
    \item \textbf{Linear Drift:} The mean $\mu_v$ is given by $\Sigma_v \nabla_T \log p_{\mathcal{M}_0}(\bar{y})$. Substituting the expanded covariance:
    \begin{equation*}
        \mathbb{E}[v \mid x] = \mu_v = (\sigma^2 I + O(\sigma^2\|n\|)) \nabla_T \log p = \sigma^2 \nabla_T \log p_{\mathcal{M}_0} + O(\sigma^2\|n\|).
    \end{equation*}
    
    \item \textbf{Geometric Drift:} For the quadratic term, we use the identity $\mathbb{E}[\Pi(v, v)] = \Pi(\mu_v, \mu_v) + \text{Tr}(\Sigma_v \Pi)$. 
    \begin{itemize}
        \item The mean contribution $\Pi(\mu_v, \mu_v)$ scales as $\|\mu_v\|^2 \sim O(\sigma^4)$.
        \item The trace contribution uses the expanded $\Sigma_v$:
        \begin{align*}
            \text{Tr}(\Sigma_v \Pi) &= \text{Tr}((\sigma^2 I + O(\sigma^2\|n\|)) \Pi) \\
            &= \sigma^2 \text{Tr}(\Pi) + \sigma^2 O(\|n\| \cdot \|\Pi\|).
        \end{align*}
        Since $\text{Tr}(\Pi) = H(\bar{y})$, this simplifies to $\sigma^2 H + O(\sigma^2\|n\|)$.
    \end{itemize}
    Combining these, the quadratic expectation is:
    \begin{equation*}
        \mathbb{E}[\Pi(v, v) \mid x] = \sigma^2 H(\bar{y}) + O(\sigma^2\|n\|) + O(\sigma^4).
    \end{equation*}
    
    \item \textbf{Higher Order Mapping Error:} Since $v \sim O_p(\sigma)$, the expectation of the manifold mapping remainder $\mathbb{E}[O(\|v\|^3)]$ scales as $O(\sigma^3)$.
\end{enumerate}

Combining all terms:
\begin{equation*}
    \mathbb{E}[Y \mid x] = \bar{y} + \sigma^2 \nabla_T \log p_{\mathcal{M}_0} + \frac{1}{2}\sigma^2 H + O(\sigma^2 \|n\|) + O(\sigma^3).
\end{equation*}

Substituting back into Tweedie's formula (\ref{eq:tweedie}):
\begin{align*}
    \nabla \log p_{\sigma}(x) &= \frac{1}{\sigma^2} \left( (\bar{y} + \sigma^2 \nabla_T + \frac{1}{2}\sigma^2 H + O(\sigma^2 \|n\| + \sigma^3)) - (\bar{y} + n) \right) \\
    &= \frac{1}{\sigma^2} \left( \sigma^2 \nabla_T + \frac{1}{2}\sigma^2 H - n + O(\sigma^2 \|n\| + \sigma^3) \right) \\
    &= -\frac{n}{\sigma^2} + \nabla_T \log p_{\mathcal{M}_0} + \frac{1}{2}H + O(\|n\| + \sigma).
\end{align*}
\end{proof}

\begin{remark}[Connection to Recent Literature]
Recent work by \cite{liu2026improvingeuclideandiffusiongeneration} (Theorem 3.1) identifies a score discrepancy term $-\frac{1}{2}\sum \frac{\partial P}{\partial x} P$. Our derivation rigorously identifies this term as the Mean Curvature Vector $H$, providing a clear geometric interpretation: the score drift is an entropic force pulling the diffusion process toward the local center of curvature. Furthermore, our result explicitly separates the normal restoration force ($-n/\sigma^2$) from the geometric drift ($H/2$), verifying that both forces act simultaneously within the tubular neighborhood defined by Assumption \ref{ass:tubular_bound}.
\end{remark}

\section{Derivation of Rectified Flow Score from VE Model}

We derive the score decomposition for the Rectified Flow model by applying a diffeomorphic transformation to the static Variance Exploding (VE) model.



\begin{lemma}[RF-VE Equivalence]
    Define the time-dependent noise scale $\sigma(t) = \frac{t}{1-t}$. The random variable $X_t$ is related to the VE variable $X_{\sigma(t)}$ by the scaling $X_t = (1-t) X_{\sigma(t)}$.
    Consequently, the score functions satisfy:
    \begin{equation*}
        \nabla_{x_t} \log p_t^{RF}(x_t) = \frac{1}{1-t} \nabla_{x_\sigma} \log p_{\sigma(t)}^{VE}(x_\sigma) \Big|_{x_\sigma = \frac{x_t}{1-t}}.
    \end{equation*}
\end{lemma}

\begin{proof}
    Let $p_\sigma^{VE}$ be the marginal density of $X_\sigma$. The density of $X_t = (1-t)X_\sigma$ is given by the change of variables formula for probability densities:
    \begin{equation*}
        p_t^{RF}(x_t) = \frac{1}{(1-t)^D} p_{\sigma(t)}^{VE}\left( \frac{x_t}{1-t} \right).
    \end{equation*}
    Taking the logarithm:
    \begin{equation*}
        \log p_t^{RF}(x_t) = \log p_{\sigma(t)}^{VE}(x_\sigma) - D \log(1-t).
    \end{equation*}
    Applying the gradient operator $\nabla_{x_t}$. By the chain rule, $\nabla_{x_t} = (\frac{\partial x_\sigma}{\partial x_t})^\top \nabla_{x_\sigma} = \frac{1}{1-t} \nabla_{x_\sigma}$. Thus:
    \begin{equation*}
        \nabla_{x_t} \log p_t^{RF}(x_t) = \frac{1}{1-t} \nabla_{x_\sigma} \log p_{\sigma(t)}^{VE}(x_\sigma).
    \end{equation*}
\end{proof}

\subsection{Proof of Proposition \ref{prop:rf-score-decomp}}
\begin{proof}
    We start with the established VE decomposition (Proposition 1) for $x_\sigma$:
    \begin{equation*}
        \nabla_{x_\sigma} \log p_\sigma^{VE}(x_\sigma) = \underbrace{-\frac{n_\sigma}{\sigma^2}}_{\text{Normal}} + \underbrace{\nabla_T \log p_{\mathcal{M}_0}(\pi_0)}_{\text{Tangent}} + \underbrace{\frac{1}{2}H_0(\pi_0)}_{\text{Curvature}} + O(\sigma).
    \end{equation*}
    We apply the Link Lemma $\nabla_{x_t} = \frac{1}{1-t} \nabla_{x_\sigma}$ and transform each term using the geometric scaling laws.

    Since $\mathcal{M}_t$ is a homothety of $\mathcal{M}_0$ by factor $(1-t)$:
    \begin{itemize}
        \item The displacement scales linearly.
        \begin{equation*}
            n_t = (1-t) n_\sigma \implies n_\sigma = \frac{n_t}{1-t}.
        \end{equation*}
        \item Curvature scales inversely with length.
        \begin{equation*}
            H_t = \frac{1}{1-t} H_0 \implies H_0 = (1-t)H_t.
        \end{equation*}
        \item The probability mass is conserved, so $p_{\mathcal{M}_t}(y_t) \propto p_{\mathcal{M}_0}(\frac{y_t}{1-t})$. The gradient scales inversely with length.
        \begin{equation*}
            \nabla_T \log p_{\mathcal{M}_t} = \frac{1}{1-t} \nabla_T \log p_{\mathcal{M}_0} \implies \nabla_T \log p_{\mathcal{M}_0} = (1-t)\nabla_T \log p_{\mathcal{M}_t}.
        \end{equation*}
    \end{itemize}

    Substitute these relations into the VE decomposition equation multiplied by $\frac{1}{1-t}$:

    \begin{itemize}
        \item {Normal Term:}
        \begin{align*}
            \frac{1}{1-t} \left( -\frac{n_\sigma}{\sigma^2} \right) &= \frac{1}{1-t} \left( - \frac{n_t / (1-t)}{(t / (1-t))^2} \right) \\
            &= \frac{1}{1-t} \left( - \frac{n_t (1-t)^2}{(1-t) t^2} \right) \\
            &= -\frac{n_t}{t^2}.
        \end{align*}
        
        \item {Tangent Term:}
        \begin{equation*}
            \frac{1}{1-t} \left( \nabla_T \log p_{\mathcal{M}_0} \right) = \frac{1}{1-t} \left( (1-t) \nabla_T \log p_{\mathcal{M}_t} \right) = \nabla_T \log p_{\mathcal{M}_t}.
        \end{equation*}
        
        \item {Curvature Term:}
        \begin{equation*}
            \frac{1}{1-t} \left( H_0 \right) = \frac{1}{1-t} \left( (1-t) H_t \right) = H_t.
        \end{equation*}
    \end{itemize}

    \paragraph{3. Conclusion.}
    Summing the transformed terms yields the result:
    \begin{equation*}
        \nabla_{x_t} \log p_t^{RF}(x_t) = -\frac{n_t}{t^2} + \nabla_T \log p_{\mathcal{M}_t}(\pi_t) + \frac{1}{2}H_t(\pi_t) + O(1).
    \end{equation*}
\end{proof}
\section{Proof of Prop 3.3}
\begin{proof}
We analyze the projection of the update step onto the normal space $N_k \equiv N_{\pi_k}\M_{t_k}$. To first order, $n_{k+1} \approx n_k + P_{N_k}(x_{k+1} - x_k)$. Substituting the update rule and score decomposition:
\begin{align*}
P_{N_k}(x_{k+1} - x_k) &= \eta_k P_{N_k}\left( -\frac{n_k}{t_k^2} + \frac{1}{2} H_{t_k} - \frac{n_k + \pi_k - (1-t_k)x_{\text{ref}}}{\sigma_p^2(t_k)} \right) \\
&= \eta_k \left( -\frac{n_k}{t_k^2} + \frac{1}{2} H_{t_k} - \frac{n_k}{\sigma_p^2(t_k)} - \frac{P_{N_k}(\pi_k - (1-t_k)x_{\text{ref}})}{\sigma_p^2(t_k)} \right).
\end{align*}
Collecting the $n_k$ terms yields the contraction coefficient $(1 - \eta_k(t_k^{-2} + \sigma_p^{-2}(t_k)))$. The remaining terms constitute the forcing constant $C_N$.
\end{proof}

\section{Proof of Prop 3.4}
\begin{proof}
We derive the evolution of the projected point $\pi_k$ by analyzing the push-forward of the ambient update rule through the metric projection map $\Pi$.

The update rule $x_{k+1} = x_k + \eta_k F(x_k)$ is driven by the total force $F(x_k)$, which combines the score and the proximal fidelity term:
\begin{equation*}
    F(x_k) = s_\psi(x_k, t_k) - \frac{x_k - (1-t_k)\xref}{\sigma_p^2(t_k)}.
\end{equation*}
Using the unique orthogonal decomposition $x_k = \pi_k + n_k$ (where $n_k \perp T_{\pi_k}\mathcal{M}_{t_k}$), we expand the fidelity term:
\begin{equation*}
    F(x_k) = s_\psi(x_k, t_k) - \frac{\pi_k - (1-t_k)\xref}{\sigma_p^2(t_k)} - \frac{n_k}{\sigma_p^2(t_k)}.
\end{equation*}

The change in the projected variable is given by the differential $d\Pi_{x_k}$ acting on the update step. By Taylor expansion:
\begin{equation*}
    \pi_{k+1} = \Pi(x_k + \eta_k F(x_k)) = \pi_k + \eta_k d\Pi_{x_k}(F(x_k)) + O(\eta_k^2).
\end{equation*}
For a point $x_k$ within the tubular neighborhood, the differential is given by $d\Pi_{x_k} = (I - \mathcal{S}_{n_k})^{-1} P_{T_k}$, where $\mathcal{S}_{n_k}$ is the Shape Operator in the normal direction $n_k$. Thus:
\begin{equation*}
    \pi_{k+1} - \pi_k \approx \eta_k (I - \mathcal{S}_{n_k})^{-1} P_{T_k}(F(x_k)).
\end{equation*}

We now compute the orthogonal tangential component $v_{tan} = P_{T_k}(F(x_k))$.
\begin{itemize}
    \item \textit{Score Component:} Using the decomposition from Proposition 3.2, $s_\psi \approx -\frac{n_k}{t_k^2} + \nabla_T \log p + H_{t_k}$. Since $n_k$ and the mean curvature $H_{t_k}$ are normal vectors, they vanish under projection:
    \begin{equation*}
        P_{T_k}(s_\psi) = \nabla_T \log p_{\mathcal{M}_{t_k}}(\pi_k).
    \end{equation*}
    \item \textit{Fidelity Component:} The normal error component $n_k/\sigma_p^2$ vanishes, leaving only the projection of the manifold displacement:
    \begin{equation*}
        P_{T_k}\left( - \frac{\pi_k + n_k - (1-t_k)\xref}{\sigma_p^2(t_k)} \right) = - \frac{1}{\sigma_p^2(t_k)} P_{T_k}(\pi_k - (1-t_k)\xref).
    \end{equation*}
\end{itemize}
Summing these gives the ideal semantic velocity $v_{tan}$.

The total tangential update is given by applying the inverse operator to the ideal velocity:
\begin{equation*}
    \frac{\pi_{k+1} - \pi_k}{\eta_k} = (I - \mathcal{S}_{n_k})^{-1} v_{tan}.
\end{equation*}
We define the drift error $\mathcal{E}_{drift}$ as the deviation from the ideal velocity $v_{tan}$:
\begin{equation*}
    \mathcal{E}_{drift} = \left( (I - \mathcal{S}_{n_k})^{-1} - I \right) v_{tan}.
\end{equation*}
Using the operator identity $(I - A)^{-1} - I = A(I - A)^{-1}$, we can rewrite the drift explicitly in terms of the shape operator:
\begin{equation*}
    \mathcal{E}_{drift} = \mathcal{S}_{n_k} (I - \mathcal{S}_{n_k})^{-1} v_{tan}.
\end{equation*}
We now apply the operator norm inequality $\|AB\| \le \|A\|\|B\|$. Recall that for the inverse operator, $\|(I - A)^{-1}\| \le \frac{1}{1 - \|A\|}$ provided $\|A\| < 1$.
\begin{equation*}
    \|\mathcal{E}_{drift}\| \le \|\mathcal{S}_{n_k}\|_{op} \frac{1}{1 - \|\mathcal{S}_{n_k}\|_{op}} \|v_{tan}\|.
\end{equation*}
Finally, substituting the curvature bound $\|\mathcal{S}_{n_k}\|_{op} \le \kappa_{max}\|n_k\|$ (valid within the tubular neighborhood where $\kappa_{max}\|n_k\| < 1$), we obtain the strict bound:
\begin{equation*}
    \|\mathcal{E}_{drift}\| \le \frac{\kappa_{max} \|n_k\|}{1 - \kappa_{max}\|n_k\|} \|v_{tan}\|.
\end{equation*}
This confirms that the drift vanishes linearly with $\|n_k\|$ (asymptotically $\approx \kappa_{max}\|n_k\|$).
\end{proof}

\section{Proof of Theorem \ref{thm:fixed_point}}

\begin{proof}
We analyze the fixed-point condition where the update step vanishes, $x_{k+1} = x_k$, which implies the total force is zero: $F(x^*) = 0$.

Recall the force decomposition from the proof of Proposition 3.4. At the fixed point $x^* = \pi^* + n^*$, the total force is:
\begin{equation*}
    F(x^*) = \underbrace{-\frac{n^*}{t^2} + \nabla_T \log p(\pi^*) + H_t(\pi^*)}_{\text{Score } s_\psi} - \underbrace{\frac{\pi^* + n^* - (1-t)\xref}{\sigma_p^2(t)}}_{\text{Fidelity Force}} = 0.
\end{equation*}

We project this equation onto the normal space $N_{\pi^*}\mathcal{M}_t$ using the projector $P_N$. Recall that $P_N(n^*) = n^*$, $P_N(\nabla_T \log p) = 0$, and $P_N(H_t) = H_t$.
The projection yields:
\begin{equation*}
    -\frac{n^*}{t^2} + H_t - \frac{n^*}{\sigma_p^2(t)} - \frac{P_N(\pi^* - (1-t)\xref)}{\sigma_p^2(t)} = 0.
\end{equation*}
Grouping the $n^*$ terms:
\begin{equation*}
    -n^* \left( \frac{1}{t^2} + \frac{1}{\sigma_p^2(t)} \right) + \underbrace{H_t - \frac{P_N(\pi^* - (1-t)\xref)}{\sigma_p^2(t)}}_{\text{Forcing } C_{net}} = 0.
\end{equation*}
Solving for the magnitude $\|n^*\|$:
\begin{equation*}
    \|n^*\| = \frac{\|C_{net}\|}{\frac{1}{t^2} + \frac{1}{\sigma_p^2(t)}} = \frac{t^2 \sigma_p^2(t)}{t^2 + \sigma_p^2(t)} \|C_{net}\|.
\end{equation*}
Since $\|C_{net}\| \le C_N$ (which is bounded due to the compactness of the manifold and smoothness of the reference projection), and $\lim_{t \to 0} \frac{t^2 \sigma_p^2}{t^2 + \sigma_p^2} = t^2$, we have $\|n^*\| = O(t^2)$. This guarantees that the final sample lies on the manifold $\mathcal{M}_0$ with vanishing error.

We project the stationarity condition onto the tangent space $T_{\pi^*}\mathcal{M}_t$ using $P_T$. Recall that $P_T(n^*) = 0$ and $P_T(H_t) = 0$.
The projection yields:
\begin{equation*}
    \nabla_T \log p_{\mathcal{M}_t}(\pi^*) - \frac{P_T(\pi^* - (1-t)\xref)}{\sigma_p^2(t)} = 0.
\end{equation*}
Rearranging terms gives the stated balance equation:
\begin{equation*}
    \nabla_T \log p_{\mathcal{M}_t}(\pi^*) = \frac{1}{\sigma_p^2(t)} P_{T_{\pi^*}}(\pi^* - (1-t)\xref).
\end{equation*}

As $t \to 0$, the manifold $\mathcal{M}_t$ converges to the clean data manifold $\mathcal{M}_0$. The term $(1-t)\xref \to \xref$, and $\sigma_p^2(t) \to \sigma_p^2$. The equilibrium condition becomes:
\begin{equation*}
    \nabla_T \log p_{\mathcal{M}_0}(\pi^*) = \frac{1}{\sigma_p^2} P_{T_{\pi^*}}(\pi^* - \xref).
\end{equation*}
We recognize the RHS as the Riemannian gradient of the Euclidean likelihood potential $U(\pi) = -\frac{1}{2\sigma_p^2}\|\pi - \xref\|^2$ restricted to the manifold. Thus, the equation is:
\begin{equation*}
    \nabla_T \log p_{\mathcal{M}_0}(\pi^*) + \nabla_T \left( -\frac{1}{2\sigma_p^2} \|\pi^* - \xref\|^2 \right) = 0.
\end{equation*}
This is precisely the first-order necessary condition for the constrained optimization problem:
\begin{equation*}
    \pi^* = \arg\max_{y \in \mathcal{M}_0} \left[ \log p(y) - \frac{1}{2\sigma_p^2} \|y - \xref\|^2 \right].
\end{equation*}
Thus, the algorithm recovers the Maximum A Posteriori (MAP) estimate of the clean image given the reference $\xref$, strictly confined to the data manifold $\mathcal{M}_0$.
\end{proof}

\section{Derivation: Score of the Conditional Flow} \label{app:rf-ode-sampler}

We aim to find the score of the marginal distribution $q_t(x_t \mid x_{\text{ref}})$ defined by the conditional generative process $X_t = (1-t)X_0 + tZ$, where $X_0 \sim p_0(\cdot \mid \xref)$:
\begin{equation*}
    q_t(x_t \mid x_{\text{ref}}) = \int p(x_0 \mid x_{\text{ref}}) \, p_t(x_t \mid x_0) \, dx_0
\end{equation*}
By Bayes' rule, substitute $p(x_0 \mid x_{\text{ref}}) = \frac{p(x_{\text{ref}} \mid x_0) p(x_0)}{p(x_{\text{ref}})}$:
\begin{equation*}
    q_t(x_t \mid x_{\text{ref}}) = \frac{1}{p(x_{\text{ref}})} \int p(x_{\text{ref}} \mid x_0) \, p(x_0) \, p_t(x_t \mid x_0) \, dx_0
\end{equation*}
Using the property of the forward diffusion process, we know that $p(x_0) p_t(x_t \mid x_0) = p_t(x_t) p(x_0 \mid x_t)$. Substituting this into the integral:
\begin{equation*}
    q_t(x_t \mid x_{\text{ref}}) = \frac{p_t(x_t)}{p(x_{\text{ref}})} \int p(x_{\text{ref}} \mid x_0) \, p(x_0 \mid x_t) \, dx_0
\end{equation*}
We assume the structural property of \textit{Conditional Independence}: $x_{\text{ref}} \perp x_t \mid x_0$. This implies $p(x_{\text{ref}} \mid x_0) = p(x_{\text{ref}} \mid x_0, x_t)$. Thus, the integral becomes the definition of the conditional likelihood $p(x_{\text{ref}} \mid x_t)$:
\begin{align*}
    \int p(x_{\text{ref}} \mid x_0, x_t) \, p(x_0 \mid x_t) \, dx_0 &= p(x_{\text{ref}} \mid x_t)
\end{align*}
Substituting back:
\begin{equation*}
    q_t(x_t \mid x_{\text{ref}}) = \frac{p_t(x_t) \cdot p(x_{\text{ref}} \mid x_t)}{p(x_{\text{ref}})}
\end{equation*}
Taking the logarithm and gradient with respect to $x_t$:
\begin{align*}
    \nabla_{x_t} \log q_t(x_t \mid x_{\text{ref}}) &= \nabla_{x_t} \log p_t(x_t) + \nabla_{x_t} \log p(x_{\text{ref}} \mid x_t) - \underbrace{\nabla_{x_t} \log p(x_{\text{ref}})}_{0} \\
    &= \nabla \log p_t(x_t) + \nabla \log p(x_{\text{ref}} \mid x_t)
\end{align*}
\textbf{Conclusion:} The sum of the unconditional score and the likelihood score exactly recovers the score of the true posterior path.

\section{RF-Inversion Analysis} \label{app:rf-inversion}

We analyze the flow driven by the mixed score field $s_{\text{mix}}(x, t)$ with weight $\eta$:
\begin{equation*}
    s_{\text{mix}}(x, t) = (\eta) \nabla \log p(x_{\text{ref}} \mid x_t) + (1-\eta) \nabla \log p_t(x_t)
\end{equation*}


\begin{theorem}[Equivalence of Hard-Guidance SGPP and RF-Inversion]
Let $\mathcal{V}_{SGPP}(x, t; \eta)$ denote the velocity field of Score-Guided Proximal Projection at time $t \in [1, 0]$ using a geometric score mixture with weight $\eta$ and proximal variance schedule $\sigma_p(t)$. Let $\mathcal{V}_{RF}(x, \tau; \eta)$ denote the control-guided velocity field of RF-Inversion at time $\tau = 1-t \in [0, 1]$ with reference $y_0$.

The fields are defined as:
\begin{align}
    \mathcal{V}_{SGPP}(x, t) &= -\frac{x}{1-t} - \frac{t}{1-t}\left[ (1-\eta)s_\theta(x,t) + \eta \nabla \log p(y_0|x) \right] \label{eq:sgpp_def} \\
    \mathcal{V}_{RF}(x, \tau) &= (1-\eta)v_\theta(x, \tau) + \eta \frac{y_0 - x}{1-\tau} \label{eq:rf_def}
\end{align}
In the hard-guidance limit where the proximal width $\sigma_{p} \to 0$, the vector fields are equivalent under time-reversal:
\begin{equation*}
    \lim_{\sigma_p \to 0} \mathcal{V}_{SGPP}(x, t; \eta) = - \mathcal{V}_{RF}(x, 1-t; \eta)
\end{equation*}
\end{theorem}

\begin{proof}
 We decompose (\ref{eq:sgpp_def}) into unconditional and conditional components:
\begin{equation*}
    \mathcal{V}_{SGPP} = (1-\eta)\underbrace{\left[ -\frac{x}{1-t} - \frac{t}{1-t}s_\theta(x,t) \right]}_{v_{unc}(x,t)} + \eta \underbrace{\left[ -\frac{x}{1-t} - \frac{t}{1-t}\nabla \log p(y_0|x) \right]}_{v_{cond}(x,t)}
\end{equation*}
The likelihood score $\nabla \log p(y_0|x)$ assumes a Gaussian forward process centered at $(1-t)y_0$. In the limit $\sigma_p \to 0$, the variance schedule $\sigma_p(t) \approx t$, yielding:
\begin{equation*}
    \lim_{\sigma_p \to 0} \nabla \log p(y_0|x) = \lim_{\sigma_p \to 0} -\frac{x - (1-t)y_0}{\sigma_p^2(t)} = -\frac{x - (1-t)y_0}{t^2}
\end{equation*}
 Substituting the limiting score into $v_{cond}$:
\begin{align*}
    v_{cond}(x, t) &= -\frac{x}{1-t} - \frac{t}{1-t} \left( -\frac{x - (1-t)y_0}{t^2} \right) \\
    &= \frac{-tx + (x - y_0 + ty_0)}{t(1-t)} = \frac{(1-t)(x - y_0)}{t(1-t)} = \frac{x - y_0}{t}
\end{align*}
Thus, the total SGPP field becomes $\mathcal{V}_{SGPP} = (1-\eta)v_{unc} + \eta (\frac{x - y_0}{t})$.

(\ref{eq:rf_def}) defines the RF conditional term as $\frac{y_0 - x}{1-\tau}$. With $\tau = 1-t$, this becomes $\frac{y_0 - x}{t}$. Since SGPP integrates backward ($dt$) and RF-Inversion forward ($d\tau$), physical equivalence requires opposite signs. Comparing the conditional forces:
\begin{equation*}
    \frac{x - y_0}{t} = - \left( \frac{y_0 - x}{t} \right)
\end{equation*}
The forces are identical in magnitude and opposite in direction, confirming the fields are equivalent.
\end{proof}
\end{document}